\documentclass[sigconf]{acmart}
\renewcommand\footnotetextcopyrightpermission[1]{} 
\usepackage{booktabs} 

\acmDOI{10.475/123_4}

\acmISBN{123-4567-24-567/08/06}

\acmConference[CONECCT]{ACM conference}{March 16-17, 2018}{Bangalore, India} 
\acmYear{2018}
\copyrightyear{2016}

\acmArticle{4}
\acmPrice{15.00}

\editor{}
\editor{}
\editor{}

\usepackage{makecell}
\usepackage[T1]{fontenc}
\usepackage{xcolor}
\usepackage{parskip}
\usepackage{amsmath,amssymb, amsthm}
\usepackage{lmodern}
\usepackage{centernot,url}
\usepackage{bm}
\usepackage{float}
\usepackage{subfigure}
\usepackage{graphicx}
\usepackage{algorithm2e}
\usepackage{algpseudocode}
\usepackage{multicol, blindtext}
\graphicspath{{/home/admin123/ldcct/plots/}}
\usepackage{multicol}
\usepackage{mathtools}

\newtheorem{definition}{Definition}
\newtheorem{thm}{Theorem}
\newtheorem{lem}[thm]{Lemma}

\DeclareMathOperator*{\argmax}{arg\,max}

\begin{document}
\title{Big Data Classification Using Augmented Decision Trees}

\author{Rajiv Sambasivan}

\affiliation{%
  \institution{Chennai Mathematical Institute}
  \streetaddress{H1, SIPCOT IT Park, Siruseri}
  \city{Kelambakkam} 
  \state{India} 
  \postcode{603103}
}
\email{rsambasivan@cmi.ac.in}

\author{Sourish Das}

\affiliation{%
  \institution{Chennai Mathematical Institute}
  \streetaddress{H1, SIPCOT IT Park, Siruseri}
  \city{Kelambakkam} 
  \state{India} 
  \postcode{603103}
}
\email{sourish@cmi.ac.in}


\begin{abstract}
\noindent\emph{ We present an algorithm for classification tasks on big data. Experiments conducted as part of this study indicate that the algorithm can be as accurate as ensemble methods such as random forests or gradient boosted trees. Unlike ensemble methods, the models produced by the algorithm can be easily interpreted. The algorithm is based on a divide and conquer strategy and consists of two steps. The first step consists of using a decision tree to segment the large dataset. By construction, decision trees attempt to create homogeneous class distributions in their leaf nodes. However, non-homogeneous leaf nodes are usually produced. The second step of the algorithm consists of using a suitable classifier to determine the class labels for the non-homogeneous leaf nodes. The decision tree segment provides a coarse segment profile while the leaf level classifier can provide information about the attributes that affect the label within a segment.}
\end{abstract}

%
%
%
%
%

\maketitle

\section{Introduction and Motivation}
Classification tasks are arguably the most common applications of machine learning. Over the years, several sophisticated techniques have been developed for classification. The size of the data set has started becoming an important consideration today in picking a method for classification. Solving the problem for linearly separable decision boundaries was an important first step \cite{zhang2004solving}. Linear decision boundaries may offer an adequate solution for some datasets but many real world classification problems are characterized by non-linear decision boundaries. Kernel methods \cite{bosern} are useful in these situations. However applying kernel methods to large datasets also has many challenges. On moderate size datasets, evaluating multiple kernels on the data and then subsequently picking hyper-parameters using a technique like grid search is a tractable approach. However with large datasets, this approach may be impractical because each experimental evaluation may be computationally expensive. Sometimes, such an iterative approach to kernel selection may not yield kernels that perform well and we may need to resort to multiple kernel learning \cite{bach2004multiple} to arrive at a suitable kernel for the problem. Since developing a single complex model for the entire dataset is a difficult task, a natural line of inquiry would be a divide and conquer strategy. This would entail developing models on segments of the data. Though ideas such as Hierarchical Generalized Linear Models \cite{lee1996hierarchical} have been developed, the method to determine the segments is a critical aspect of such an approach. Recently we reported a method to perform big data regression using a Classification and Regression Tree (CART) \cite{breiman1984classification} to perform this segmentation \cite{sambasivan2017big} . The effectiveness of this approach with regression problems suggested that this technique could be applied to classification tasks as well.\\
Experiments reported in this study suggest that this approach could be effective for classification tasks as well. The approach is characterized by two steps. The first step uses a CART decision tree to segment the large dataset. The leaves of the decision tree represent the segments. Decision trees minimize an impurity measure like the misclassification error, gini-index \cite{gini1971variability} or the cross-entropy at the leaves. While some leaves may be almost homogeneous with respect to the class distribution, in a large dataset a decision tree that generalizes well may have many leaves where the class distribution is not homogeneous. These nodes may  require a classifier that can determine complex decision boundaries or these leaves may represent noisy regions of the data . Accordingly, the second step of the algorithm fits a classifier to those nodes where the class distribution is non-homogeneous. In the experiments reported in this study we found that it was possible to increase classification accuracies in some cases. When this strategy fails we observed that this was because all classifiers perform poorly at certain leaf nodes. This suggests that these nodes are either noisy or may require additional features to achieve good classification performance. In this study, we observed this behavior with the census income dataset (see section \ref{sec:datasets} for details of the dataset). The classification task for this dataset is to predict the income level for an individual given socio-economic features. In segments of poor performance we found records of individuals working a high number of hours per week with the state government jobs, but reporting a low income. These records seem to be of dubious quality, since even with minimum-wage, these instances should belong to the higher income category. When these noisy segments were removed, we are able to enhance accuracy. Therefore this algorithm either achieves good accuracies or it helps us identify potentially noisy or difficult regions of our dataset. An attractive feature of this algorithm is the ease with which the resulting models can be interpreted. For any data instance, the decision tree model yields the aggregate properties associated with that instance. The leaf level model obtained from the second step can then be interpreted to yield insights into factors that affect the decision for that leaf. In the experiments conducted as part of this study we found that the accuracy of the proposed approach matches what is obtained with ensemble methods like gradient boosted trees \cite{chen2016xgboost} or random forests \cite{breiman2001random}. Models produced by ensemble methods are difficult to interpret in contrast to the models produced by the proposed method. Therefore the proposed method can produce models that are both interpretable and accurate. This is highly desirable.\\

\section{Problem Context}\label{sec:pc}
\noindent We are given a dataset $\mathcal{D}$, let $x_i$ represent the predictor variables and $y_i$ represent the label associated with the instance $i$. Observations are ordered pairs $(x_i, y_i),\ i = 1,2,\hdots, \mathit{N}$. Class labels $y_i$ are represented by $\{0,1,\hdots, K-1\}$. Classification trees partition the predictor space into $m$ regions, $\mathit{R_1},\mathit{R_2}, \hdots, \mathit{R_m}$. Consider a $\mathit{K}$ class classification problem. For a leaf node $m$, representing region $\mathit{R_m}$ with $\mathit{N_m}$ observations, the proportion of observations belonging to class $\mathit{k}$ is defined as:
\begin{equation*}
\mathit{\hat{p}_{mk}} = \frac{1}{\mathit{N_m}} \sum_{x_i \in \mathit{R_m}} \mathit{I}(y_i = k),
\end{equation*} 
$\mathit{I}(y_i = k) =  \begin{cases*} 1\quad if\ y_i = k\\ 0,\ otherwise. \end{cases*} $

\noindent CART labels the instances in leaf node $m$ with label $$k(m) = \underset{ k \in \mathit{K}} \argmax \quad \mathit{\hat{p}_{mk}}, $$ see \cite{friedman2001elements}[Chapter 9, section 9.2]. During tree development, CART tries to produce leaf nodes that are as homogeneous as possible. Typically not all leaf nodes are homogeneous. Leaves that are non-homogeneous with respect to the class distribution are data regions where we can enhance the performance of the decision tree.  Section \ref{sec:dt_for_seg} provides the details of the how this is achieved.
\section{Decision Trees For Segmentation}\label{sec:dt_for_seg}

\begin{algorithm}[ht]
\KwData{Dataset $\mathcal{D}$, leaf size $l_s$ and Test Dataset $\mathcal{T}$}
\KwResult{Segmented Decision Tree Model}
\SetKwData{SM}{seg.model}
\SetKwData{SRM}{seg.cl.model}
\SetKwData{TS}{test.seg}
\SetKwData{TRM}{test.cl.model}
\SetKwData{TRP}{pred.value}
\SetKwFunction{BSM}{Fit.Seg.Model}
\SetKwFunction{BSRM}{Fit.Seg.Classifier.Models}
\SetKwFunction{GSR}{Seg.Classifier.Model}
\SetKwFunction{PTR}{Predict}

\tcc{Fit Segmentation Model}
\SM$\leftarrow$ \BSM{$\mathcal{D}, l_s$}\;
\For{Each Segment $s$ in $\mathcal{D}$}{
\tcc{Fit Segment Classifier Model - can be a sophisticated model because segment size is small}
\tcc{A pool of classifiers is developed and the best leaf classifier for the segment on the training data is noted}
\SRM$\leftarrow$ \BSRM{$s$}\;
}
\tcc{Score Test Set}
\For{Each record $r$ in $\mathcal{T}$}{
\tcc{Get Segment for Test Record from Decision Tree}
\TS$\leftarrow$ \SM{$r$}\;
\tcc{Get best classifier model for \TS}
\TRM$\leftarrow$ \GSR{\TS}\;
\tcc{Obtain prediction for $r$ from \TRM}
\TRP$\leftarrow$ \PTR{\TRM, $r$}\;
}

\caption{Big Data Classification Using Decision Tree Segmentation}\label{algo:dt_tree_seg_reg}
\end{algorithm}

\noindent The first step of the algorithm is to segment the dataset with a CART decision tree. The second step of the algorithm is to augment the performance of the decision tree classifier in segments or leaves where the class distribution is non-homogeneous. This is achieved by using a suitable leaf level classifier. A pool of classifiers is developed for these segments and the best performing classifier, as indicated by the cross-validated training error is used as the leaf classifier for the segment. Algorithm \ref{algo:dt_tree_seg_reg} summarizes these ideas. The number of instances at the leaf  or equivalently the height of the decision tree is an important parameter. The following factors need to considered in picking this parameter:
\begin{enumerate}
\item Generalization error of the decision tree: We need to avoid over-fitting the decision tree model. The decision rules produced by the tree should be valid for the test set and produce a test error that is not very different from the training error.
\item Total generalization error of the algorithm: We want our algorithm to be as accurate as possible. The leaf size at which the best decision tree error is obtained may be different from the leaf size at which the lowest overall error is obtained for the algorithm. We need to ensure that the composite model generalizes well. 
\end{enumerate} 
These ideas are discussed and illustrated in section \ref{sec:experiments}.

\section{Leaf Classifiers}\label{sec:leaf_classifiers}
\noindent The key idea with the algorithm presented in this work is to augment the performance of decision tree nodes where the class distribution is non-homogeneous. Using a suitable classifier, we may be able to determine decision boundaries in these segments that result in better classification accuracy than what is produced with the plain decision tree. This strategy works very well for some datasets. Sometimes however we do encounter nodes where all classifiers perform poorly. This typically happens for a small proportion of the segments. These segments are probably noisy or require additional features for achieving good classification performance. Strategies to deal with these segments are discussed in section \ref{sec:dor}. \cite{kohavi1996scaling} presents an algorithm called NBTree that is similar to the idea presented in this work. \cite{kohavi1996scaling} work uses a Naive Bayes classifier for the leaf nodes.  The tree algorithm used in \cite{kohavi1996scaling} is C4.5 \cite{quinlan2014c4}. In this work we used the CART \cite{breiman2001random} algorithm for the decision tree. The accuracies obtained with a decision tree based on the CART algorithm is higher than what is reported with NBTree in \cite{kohavi1996scaling}, see section \ref{sec:accuracy} for details. In  \cite{kohavi1996scaling}, the Naive Bayes classifier is the leaf classifier for every leaf whereas this algorithm permits flexibility with this decision. We can pick any classifier for the leaf node. Further, we can fit a pool of classifiers and then pick the best classifier for a node based on cross-validation scores observed during training. It is also important to note that leaf nodes may homogeneous with respect to the class distribution, in which case there is no need to fit a classifier. We can accept the decision tree results for that leaf. The implementation of Algorithm \ref{algo:dt_tree_seg_reg} that was used for this study implements these features.

\section{Feature Selection}\label{sec:eol}
\noindent The datasets used for this study were high dimensional. These public datasets have also been used research and in kaggle competitions \cite{kaggle}. Feature selection performed using the extremely randomized trees algorithm \cite{geurts2006extremely} helps in removing noisy features for some of the datasets used in this study. Extremely randomized trees is a tree based algorithm that is similar to random forests, but with some important differences. Unlike in random forests where the splits are determined on the basis of an impurity measure like gini \cite{gini1971variability} or cross-entropy, the splits in extremely randomized trees are randomly determined.

\section{Methodology}\label{sec:methodology}
\noindent The datasets used in this study have been featured in kaggle competitions. A review of the competition forums reveal that feature selection and ensemble tree based algorithms are the key components for good performance on these datasets. The methodology used to evaluate the effectiveness of the algorithm proposed in this study is as follows. Feature selection is performed to determine the relevant features for a dataset. If feature selection did not help improve accuracy, we retain all original features. We use a CART decision tree on the datasets used for this study and obtain a performance measurement. Next we apply the random forest and gradient boosted trees algorithm on the datasets and obtain a performance measurement. Finally we apply Algorithm \ref{algo:dt_tree_seg_reg} and obtain the performance measurement for the proposed algorithm. We can then evaluate how the accuracy obtained with the proposed algorithm compares with ensemble methods such as random forests or gradient boosted trees. We can also evaluate the accuracy gain obtained with Algorithm \ref{algo:dt_tree_seg_reg} over a plain CART decision tree.
It should be noted that the leaf size is an important parameter in applying Algorithm \ref{algo:dt_tree_seg_reg}. The leaf size used was one that produced good accuracy and good generalization. In general, the decision tree generalizes well at this leaf size. Experiments that illustrate the effect of the leaf size parameter are discussed in section \ref{sec:experiments}. 
\section{Experiments} \label{sec:experiments}
\subsection{Datasets} \label{sec:datasets}
\noindent The datasets for this study all came from the UCI data repository \cite{Lichman} and the US Department of Transportation website \cite{RITA_Delay_Data_Download}. These datasets also figure in the kaggle playground category competitions. Playground category competitions are organized for the purpose of testing out machine learning ideas \cite{no_free_hunch_2016}. There is no prize money involved. The following datasets were used for this study:
\begin{enumerate}
\item \textbf{Forest Cover Type}: This dataset is used to predict the type of forest cover given cartographic information. The data covers four wilderness areas in the Roosevelt National Forest of northern Colorado. The dataset consists of 581012 instances with 54 attributes. There are no missing values for attributes.
\item\textbf{ Airline Delay}: This dataset is used to predict airline travel delays. The dataset is obtained from the US Department of Transportation website \cite{RITA_Delay_Data_Download}. The data consists of flight on time arrival performance for the months of January and February of 2017. A flight is considered delayed if it is associated with an arrival delay of fifteen minutes or more. Thirteen flight information attributes are extracted. The dataset has over eight hundred thousand records.
\item \textbf{Census Income}: This dataset contains data extracted from the 1994 census database. The prediction task associated with this dataset is to predict if the annual income of an individual is over 50,000 dollars or not. The dataset has missing attributes. This dataset has been studied in a machine learning context in \cite{kohavi1996scaling}. As done in \cite{kohavi1996scaling}, we have ignored records with missing values. Unlike the previous datasets, feature selection did not improve accuracy with this dataset and all original attributes were retained for analysis. This dataset has 14 attributes and 45220 complete instances (rows with no missing values).
     
\end{enumerate}
.
\subsection{Experimental Evaluation of Leaf Size}
As discussed in section \ref{sec:dt_for_seg} and section \ref{sec:methodology}, the leaf size (or equivalently the tree height) is an important parameter for the algorithm presented in this work. The leaf size can affect:
\begin{enumerate}
\item The generalization of the decision tree used to segment the data.
\item The generalization of the overall model.
\end{enumerate}
Therefore we need two experiments. The first experiment illustrates the effect of the leaf size on the generalization of the decision tree model. The second experiment illustrates the effect of the leaf size on the overall model (Algorithm \ref{algo:dt_tree_seg_reg}). For both these experiments, 70 percent of the data was used for training and 30 percent of the data was used for the test set. For both these experiments, the classification accuracy was used as the metric for error. All modeling for this study was done in \texttt{Python} with the \texttt{scikit-learn}(\cite{scikit-learn}) library.

\section{Discussion of Experimental Results}\label{sec:dor}
\noindent The key idea with the proposed algorithm is that we can enhance decision tree classification performance in leaf nodes where the decision tree is unable to produce homogeneous class distributions. It is possible that other classification techniques like kernel methods or K-Nearest Neighbors may be able to identify good decision boundaries in these nodes. In the experiments reported in this study we found one of the following two types of behavior:
\begin{enumerate}
\item Leaf Classifiers can enhance classification performance: This was the case with the forest cover type identification dataset and the airline delay dataset. We could enhance classification accuracy at the leaf nodes by using another classifier.
\item Leaf Classifiers are unable to enhance classification performance: In this case all classifiers perform poorly in certain regions of the dataset. This kind of behavior was noted with the census income dataset.
\end{enumerate}

\begin{figure}[ht]

\includegraphics[width= \linewidth]{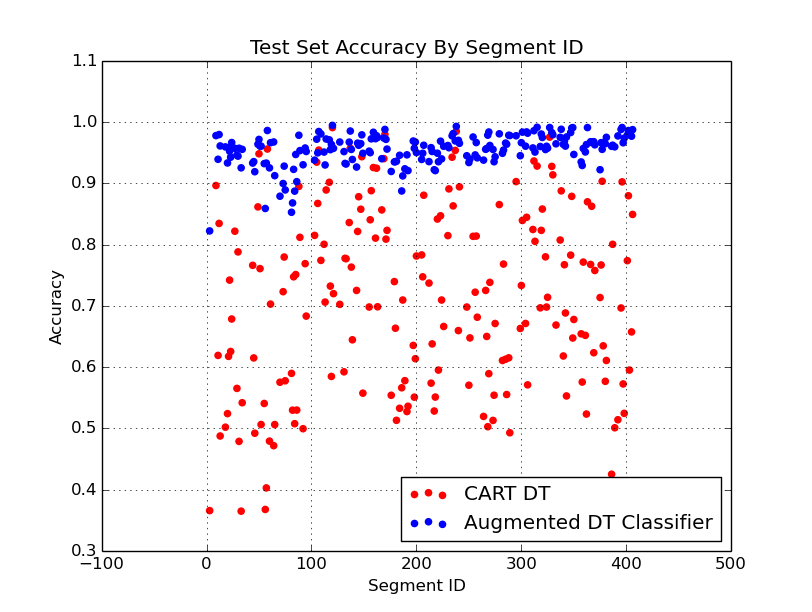}\par\caption{Segment Accuracy Enhancement - Forest Cover Type}
\label{fig:EXP_SAE_FC}
\end{figure}

\begin{figure}
\includegraphics[width=\linewidth]{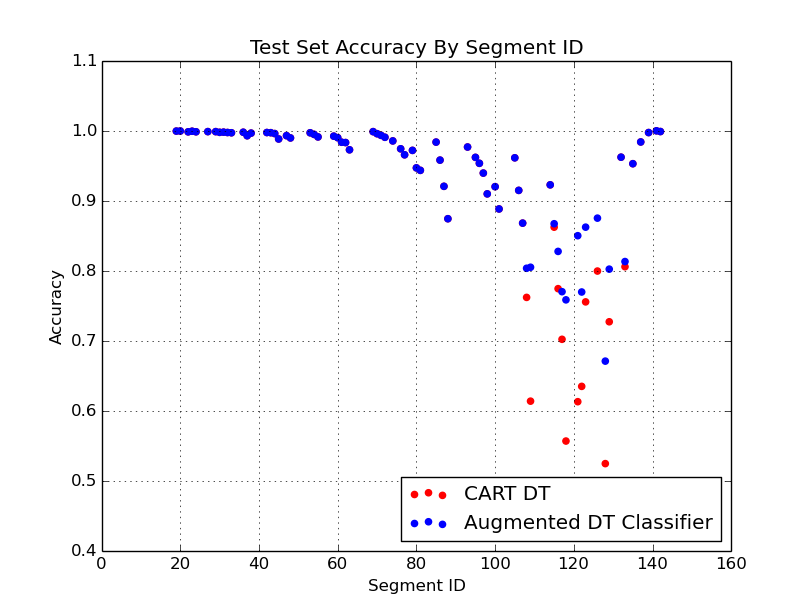}\par\caption{Segment Accuracy Enhancement - Airline Delay}
\label{fig:EXP_SAE_AD}
\end{figure}

\begin{figure}[ht]
\includegraphics[width= \linewidth]{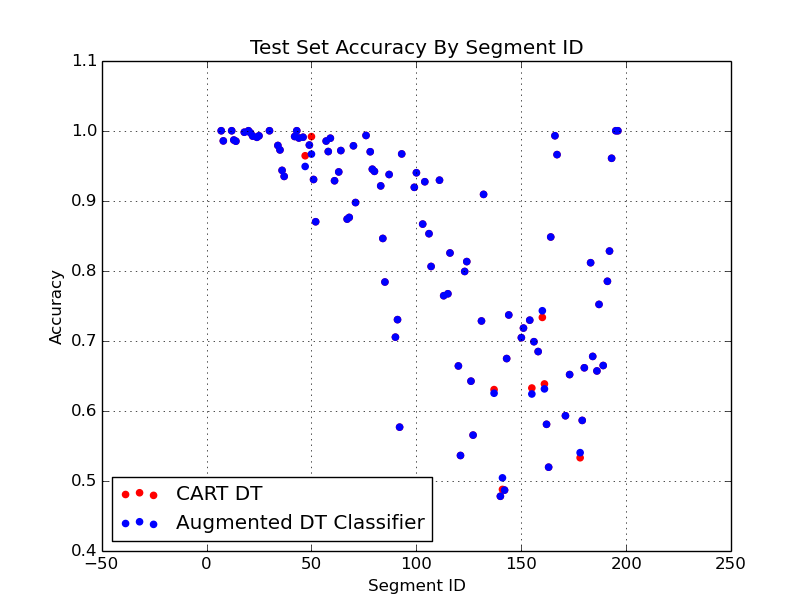}\par\caption{Segment Accuracy Non Enhancement- Census Income}
\label{fig:EXP_SAE_CI}
\end{figure}

An illustration of the effectiveness of leaf level classifiers is provided in Figure \ref{fig:EXP_SAE_FC} through Figure \ref{fig:EXP_SAE_CI}. These plots illustrate the increase in accuracy of test set prediction when using a leaf level classifier. The accuracies obtained with a plain decision tree are illustrated in red while the accuracies obtained when a leaf level classifier is used is illustrated in blue. The leaf size for these experiments are the values at which both good accuracy and good generalization are observed . With the forest cover identification dataset, the leaf level classifiers are able to enhance the accuracy in almost every segment of the dataset. The leaf level classifier that was effective in almost all leaf segments of the forest cover dataset was the KNN (K Nearest Neighbors) classifier with a window size of 3 neighbors. In this experiment we used a neighborhood size of 3 for all leaves. It is possible that this neighborhood size is non-optimal for some segments. So we could possibly enhance the accuracy reported in section \ref{sec:accuracy} by tuning this parameter in segments in the lower end of the accuracy range.\\
The response for the airline delay dataset indicates if a flight is going to delayed over 15 minutes. The response for this dataset is skewed. The proportion of delays for the test set is shown in Figure \ref{fig:SEG_DELAYED_PROP_AD}. As is evident, flight delays are fairly uncommon for most segments. However there are a small proportion of segments characterized by higher delays. The segment ID's for these higher delay segments range from 100 through 140. An analysis of Figure \ref{fig:EXP_SAE_AD} shows that leaf level classifiers help enhance accuracies in these segments. It appears that there are more blue points than red points in Figure \ref{fig:EXP_SAE_AD}. Many segments have very low proportion of delays. In these segments, there is really no advantage in using a leaf level classifier. The plain CART decision tree does well in regions where the response is fairly homogeneous (very low delays). The accuracies of the leaf level classifier and the plain decision tree overlap in these regions. The leaf level classifier accuracy is plotted second, so there is more blue evident. In the higher delay segments the leaf level classifiers perform better than the plain decision tree and therefore the difference is evident in such regions. For the higher delay segments there is no classifier that performs best for all the segments. For some segments a logistic regression classifier performed best, while for others an SVM classifier or a KNN classifier performed best. As with the forest cover dataset, there is scope for improvement of the accuracy reported in section \ref{sec:accuracy} by fine tuning the classifier hyper-parameters in the higher delay segments.\\

The census income dataset provides an example of where leaf level classifiers do not help in enhancing accuracy. An analysis of Figure \ref{fig:EXP_SAE_CI} shows that there are many segments in the segment ID range 100 through 200 where the accuracy of prediction is low. For example, there are many segments where the accuracy of prediction is less than 60\%. An analysis of the classifier accuracies in these regions revealed that all classifiers perform poorly in these problematic segments. The accuracy obtained with the plain decision tree is the same as the accuracy obtained with leaf level classifiers in these segments. This suggests that these regions are noisy or that we may require a better set of features for these regions. This is discussed in section \ref{sec:accuracy}. As with the airline delay dataset, there is a lot of overlap in the accuracies produced by the plain decision tree and the leaf level classifiers. However, the performance is characterized by two regions. A region where the decision tree and the augmented decision tree perform well and a region where the the decision tree and the augmented decision tree perform poorly. 

In summary, we are either able to enhance performance or we are able to identify problematic regions of our dataset when we use Algorithm \ref{algo:dt_tree_seg_reg}. Problematic regions are those where all classifiers perform poorly. This data could be isolated for further analysis. Removing these problematic regions enhances accuracy (see section\ref{sec:accuracy}).
 
\subsection{Effect of Leaf Size on Decision Tree Generalization}\label{sec:eff_ldct}
These experiments illustrate the effect of the leaf size parameter on the generalization error of the decision tree. Tree growth along a particular path in the tree is stopped when the number of instances in the node falls below a threshold level. The training error and the test set error associated with the leaf size setting is noted. This procedure is repeated for various values of the threshold level of the leaf size parameter. The results are shown in Figure \ref{fig:EXP_LSDT_AD} through Figure \ref{fig:EXP_LSDT_FC}.

\begin{figure}[ht]
\includegraphics[width=\linewidth]{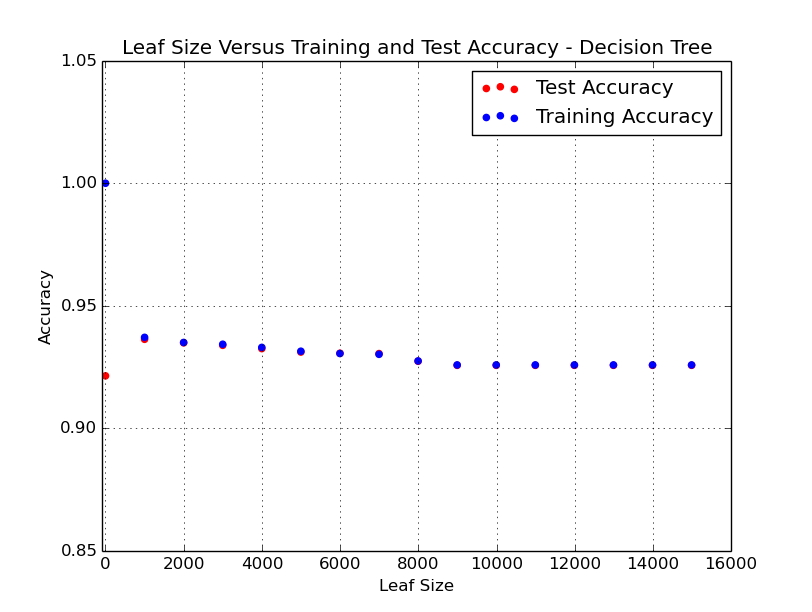}\par\caption{Airline Delay Decision Tree Generalization}
\label{fig:EXP_LSDT_AD}
\end{figure}

\begin{figure}[ht]
\includegraphics[width= \linewidth]{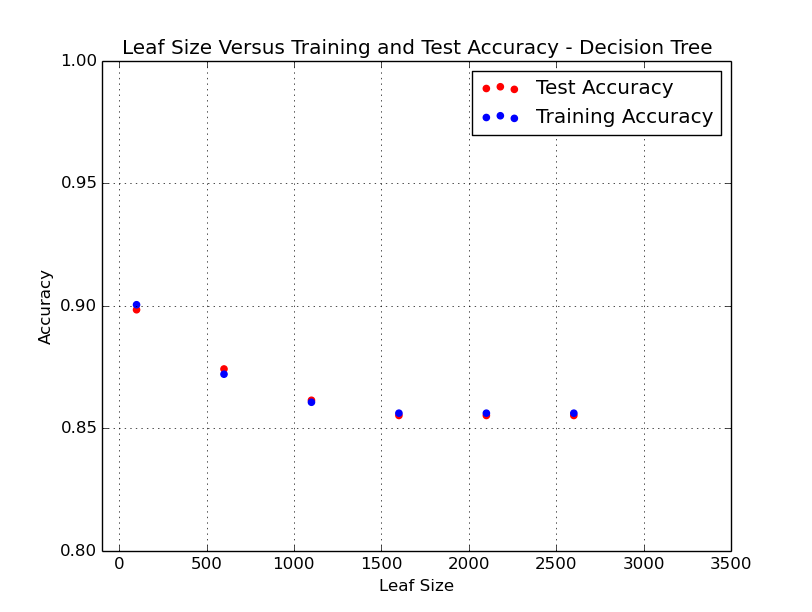}\par\caption{Census Income Decision Tree Generalization}
\label{fig:EXP_LSDT_CI}
\end{figure}

\begin{figure}
\includegraphics[width= \linewidth]{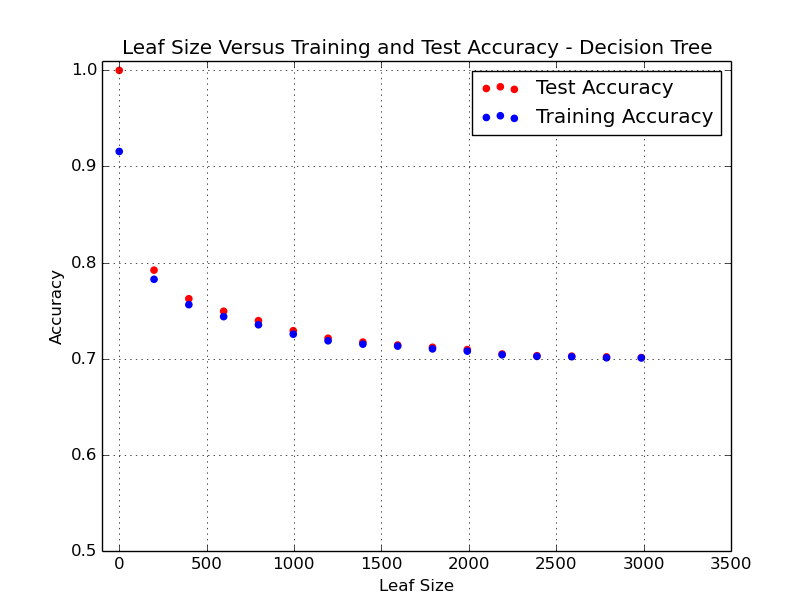}\par\caption{Forest Cover Type Decision Tree Generalization}
\label{fig:EXP_LSDT_FC}
\end{figure}
As is evident, the generalization error of the decision tree is good over the entire range of leaf sizes. The single exception is the case of using a leaf size of one for the forest cover and airline delay datasets. As expected, there is significant over-fitting for this case.

\subsection{Effect of Leaf Size on Model Generalization}
These experiments illustrate the effect of the leaf size on the error of Algorithm \ref{algo:dt_tree_seg_reg}. For each leaf size, the training and test error are noted. The results of these experiments are shown in Figure \ref{fig:EXP_LSMODEL_AD} through \ref{fig:EXP_LSMODEL_FC}.
\begin{figure}[ht]
\includegraphics[width=\linewidth]{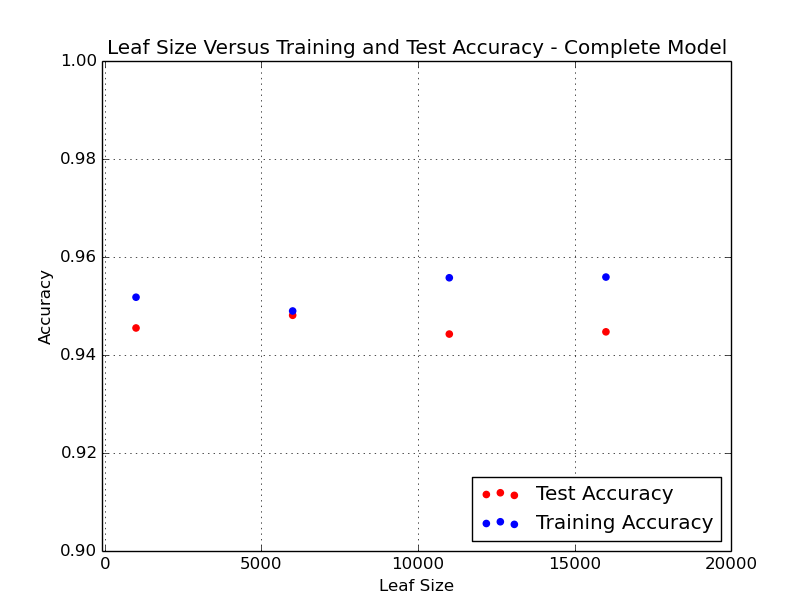}\par\caption{Airline Delay Overall Model Generalization}
\label{fig:EXP_LSMODEL_AD}
\end{figure}

\begin{figure}
\includegraphics[width= \linewidth]{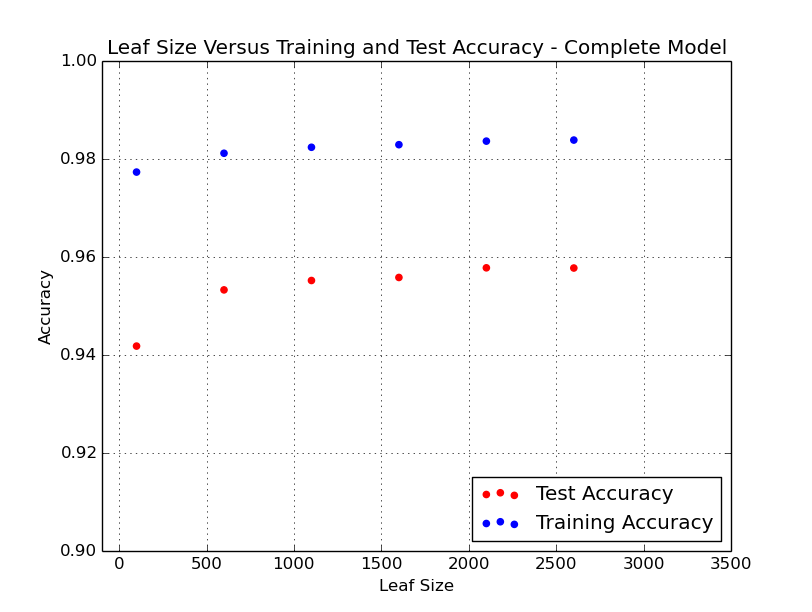}\par\caption{Forest Cover Type Overall Model Generalization}
\label{fig:EXP_LSMODEL_FC}
\end{figure}
The optimal leaf size is one where we achieve good accuracy and good generalization. For the airline delay dataset this optimal value is about 6000 (see Figure \ref{fig:EXP_LSMODEL_AD}). For the forest cover dataset, the optimal value is around 1500 (see Figure \ref{fig:EXP_LSMODEL_FC}). There is very little accuracy gain with increasing the leaf size beyond 1500 for the forest cover dataset. For both the forest cover and the airline delay dataset, the decision tree generalizes well at the optimal settings (see Figure \ref{fig:EXP_LSDT_AD} and Figure \ref{fig:EXP_LSDT_FC}).

\subsection{Accuracy}\label{sec:accuracy}
\noindent As discussed in section \ref{sec:methodology}, we can evaluate the effectiveness of the algorithm by comparing the accuracy obtained with Algorithm \ref{algo:dt_tree_seg_reg} with those obtained from ensemble methods like random forests or gradient boosted trees. We can also evaluate the improvement in accuracy over using a plain decision tree. The accuracies obtained with a plain decision tree are shown in Table \ref{tab:baseline_accuracy}. The leaf size used with the plain decision tree is one where the best accuracy was observed. The sizes of the datasets are provided in section \ref{sec:datasets}. For all experiments, 70\% of the dataset was used for training and 30\% was used as the test set.

%
\begin{table}[ht]
\centering
\begin{tabular}{rlrr}
  \hline
 & Dataset & Leaf.Size & Accuracy \\ 
  \hline
1 & Forest Cover &  1 & 0.916 \\ 
  2 & Airline &  1000 & 0.937 \\ 
  3 & Census Income &  100 & 0.853 \\ 
   \hline
\end{tabular}
\caption{Baseline Accuracies - CART Decision Tree Algorithm}\label{tab:baseline_accuracy}
\end{table}
Algorithm \ref{algo:dt_tree_seg_reg} can enhance the baseline accuracies reported in Table \ref{tab:baseline_accuracy} for the forest cover and the airline delay datasets. The improvement in accuracies are reported in Table \ref{tab:algorithm_accuracy}. Ensemble methods also achieve high accuracies for these datasets, however the models they produce are not interpretable. Algorithm \ref{algo:dt_tree_seg_reg} produces models that are very easily interpretable. 

\begin{table}[ht]
\centering
\scriptsize
\begin{tabular}{rllr}
  \hline
 & Dataset & Method & Accuracy \\ 
  \hline
1 & Airline Delay & XGBoost & 0.944 \\ 
  2 & Airline Delay & Random Forest & 0.946 \\ 
  3 & Airline Delay & \makecell{DT Segmented Classifiers, leaf size = 6000} & 0.945 \\ 
  4 & Forest Cover & XGBoost & 0.936 \\ 
  5 & Forest Cover & Random Forest & 0.945 \\
  6 & Forest Cover & \makecell{DT Segmented Classifiers, leaf size = 1500} & 0.957 \\  
   \hline
\end{tabular}
\caption{Accuracies obtained with Algorithm \ref{algo:dt_tree_seg_reg}}\label{tab:algorithm_accuracy}
\end{table}
\normalsize
With the census income dataset we observed that we have a small proportion (13.58\% of the data set instances) of decision tree nodes where all classifiers perform poorly. A preliminary analysis of these records indicates that there are possible data errors with this segment. For example there are records of people working with the state government working over 70 hours a week but reporting less than 50 thousand dollars in income. These records seem dubious since even at minimum wage, such employees should make over fifty thousand dollars. In any case these set of records may require further analysis to determine if they are either noisy or require additional features to obtain better classification performance.\\
When these records are removed from the dataset, we obtain an accuracy of 90.04\% (see Figure \ref{fig:EXP_LSDT_CI}). Algorithm \ref{algo:dt_tree_seg_reg} provides us a method to identify such problematic regions of our dataset. \cite{xiong2006enhancing} provides techniques to remove noise from datasets. Finding the noisy regions in large datasets and separating them from regions of good data quality is a time consuming task. Algorithm \ref{algo:dt_tree_seg_reg} can help identify these regions. Noise removal techniques, such as those discussed in \cite{xiong2006enhancing} can then be applied to see if these can help improve classification accuracy. Therefore, there is scope for improving the accuracy with the census income dataset as well.
\cite{kohavi1996scaling} report an accuracy of 84.47\% with the NBTree algorithm. Training and test sizes used in \cite{kohavi1996scaling} are similar to those used in this work. A review of Table \ref{tab:baseline_accuracy} shows that the baseline accuracy with a CART decision tree is 85.3\%. \cite{kohavi1996scaling} reports an accuracy of 81.91\% for a C 4.5 decision tree. This suggests that the choice of the decision tree (C 4.5 versus CART) can affect the accuracy.

\subsection{Interpreting the Model}\label{sec:interpretation}
\noindent Models produced by Algorithm \ref{algo:dt_tree_seg_reg} have a simple interpretation. A data instance can be associated with two models - the segment model and the leaf classification model. The segment model provides an aggregate profile for the data instance while the leaf classification model can yield insights into the factors that affect the label for an instance within the segment. Therefore we can interpret the model at coarse and fine granularities. A sample of the segment profiles for the forest cover dataset is shown in Table \ref{tab:sample_forest_cover_seg}. The columns provide the relative proportion of the different types of tree cover in that segment. It is clear that each segment is characterized by a particular set of tree cover. 

\begin{table}[ht]
\centering
\scriptsize
\begin{tabular}{rrrrrrrrr}
  \hline
 & Seg. ID & CT\_1 & CT\_2 & CT\_3 & CT\_4 & CT\_5 & CT\_6 & CT\_7 \\ 
  \hline
1 &    3 & 0.003 & 0.085 & 0.262 & 0.361 & 0.003 & 0.287 & 0.000 \\ 
  2 &   10 & 0.000 & 0.015 & 0.785 & 0.002 & 0.000 & 0.197 & 0.000 \\ 
  3 &   11 & 0.000 & 0.032 & 0.653 & 0.033 & 0.000 & 0.281 & 0.000 \\ 
  4 &   12 & 0.000 & 0.026 & 0.885 & 0.011 & 0.000 & 0.078 & 0.000 \\ 
  5 &   15 & 0.000 & 0.000 & 0.437 & 0.017 & 0.000 & 0.546 & 0.000 \\ 
   \hline
\end{tabular}
\caption{Sample of Segment Profiles - Forest Cover Dataset}\label{tab:sample_forest_cover_seg}
\end{table}
\normalsize  
Similarly, Figure \ref{fig:SEG_DELAYED_PROP_AD} shows the proportion of flights delayed by segment ID. It is evident that some segments are associated with a higher proportion of delays while many segments have very low proportion of delays. Most decision tree implementations provide a feature to generate decision rules or tree visualizations. The decision tree visualization for the airline delay dataset is shown in Figure \ref{fig:DT_VISUAL_AD}. The leaves are color coded to indicate the majority class for that node. The blue nodes indicate the nodes associated with delays. This provides an easy way to generate the coarse grained profile for a segment. We can then interpret the leaf level model for an instance, for example with the airline delay dataset, a logistic regression model, to determine the factors that affect flight delays for a particular segment. In summary the models provided by the algorithm reported in this work can be easily interpreted. This is in contrast to ensemble models like Random Forests or Gradient Boosted Trees. While these can provide accurate predictions, the models they produce are not interpretable.
\begin{figure}[ht]
\includegraphics[width= \linewidth]{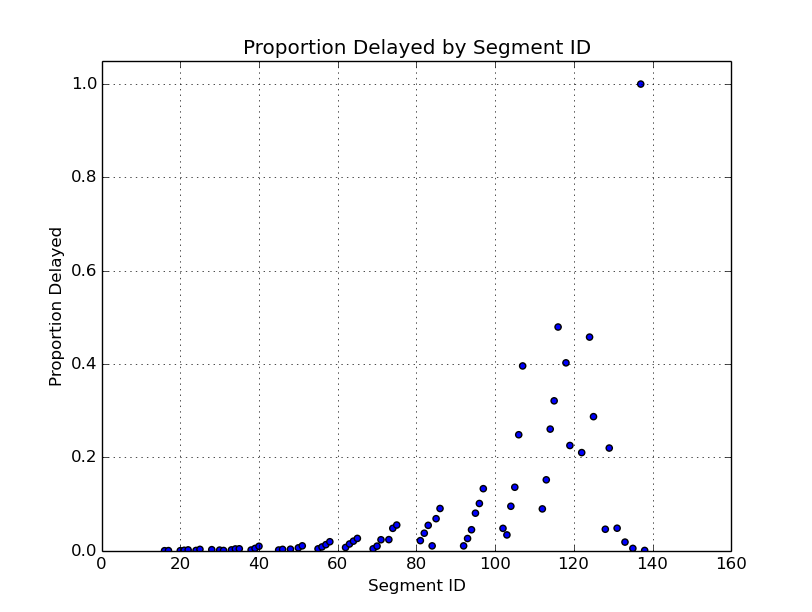}\par\caption{Airline Delay Segment Delayed Proportion}
\label{fig:SEG_DELAYED_PROP_AD}
\end{figure}

\begin{figure}[ht]
\includegraphics[scale = 0.5]{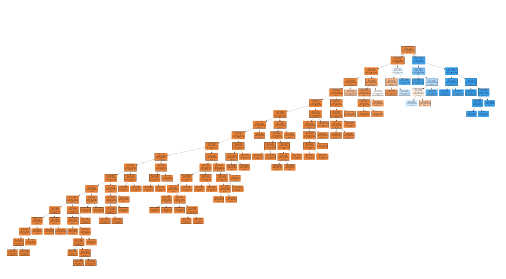}\par\caption{Airline Delay Decision Tree Visualization}
\label{fig:DT_VISUAL_AD}
\end{figure}

\section{Analysis of Segment Classifiers}

\subsection{Generalization of Segment Classifiers}

The effect of the leaf size on the generalization of the decision tree and Algorithm 1 was evaluated experimentally. A related concern is the generalization of a particular segment classifier. We can analyze the generalization of the classifier for a particular segment using concentration inequalities. The segment classifier is a function $f:x\rightarrow y$. Here $x\in \mathcal{X}=\mathbb{R}^d$ represents the predictor variables and $y\in \mathcal{Y}=\{0,1,...,k-1\}$ represents the label. The `0-1' loss function $l$ is defined as
$$
l(f(x),y)=\bigg\{\begin{array}{cc}
1 & \text{, if } f(x)\neq y,\\
0 & \text{, if } f(x) = y.
\end{array}
$$
Ideally, we want to learn the function by minimizing the risk of misclassification, where the misclassification risk is the expected value of the loss function over the joint density of the data. The joint density of the data, $\mathbb{P}(x,y)$ is defined over $\mathcal{X}\times\mathcal{Y}$. 

\begin{definition}[Misclassification Risk]
The statistical misclassification risk for the classifier $f$ is defined as 
$$
R(f)=\mathbb{E}_{\mathbb{P}}[l(f(x),y)]=\int l(f(x),y)d\mathbb{P}(x,y).
$$
\end{definition}

The joint density of the data for a segment , $\mathbb{P}(x,y)$, is however not known in practice. What we have access to is the training data. Therefore in practice the loss of the classifier is evaluated over the training data. This yields the empirical misclassification risk.

\begin{definition}[Empirical Misclassification Risk]
The empirical misclassification risk for the classifier $f$  is defined as
$$
\hat{R}_n(f)=\mathbb{E}_{\hat{\mathbb{P}}}[l(f(x),y)]=\frac{1}{n}\sum_{i=1}^nl(f(x_i),y_i)
$$

\end{definition}

\begin{lem}[Concentration Inequality for Segment Misclassification Error.] \label{lemma:sample_size}
For a given $n \geq \frac{1}{2\epsilon^2}\log(\frac{2}{\delta})$ and $0<\delta<1$, 
$$
\mathbb{P}(|\hat{R}_n(f)-R(f)|<\epsilon)>1-\delta.
$$
\end{lem}

\begin{proof}

Hoeffding's inequality \cite{hoeffding_63} states that  
\begin{quote}
If $Z_1, Z_2,\hdots,Z_n$ are independent with $\mathbb{P}(a\leq Z_i \leq b)=1$ and have a common mean $\mu$ then
$$
\mathbb{P}(|\bar{Z}-\mu|>\epsilon)<\delta
$$
where $\bar{Z}=\frac{1}{n}\sum_{i=1}^nZ_i$ and $\delta=2\exp\{-\frac{2n\epsilon^2}{(b-a)^2}\}$.
\end{quote}
In our case, we define $Z_i=l(f(x_i),y_i)$, which is bounded with probability one ($a = 0\ and\ b = 1$) and have common mean $R(f)$. When we set $n=\frac{1}{2\epsilon^2}\log\big(\frac{2}{\delta}\big)$, we have
\begin{eqnarray*}
\mathbb{P}[|\hat{R}_n(f)-R(f)|\geq\epsilon]&\leq& 2\exp\{-2n\epsilon^2\},\\
&=&2 \exp\{-2\frac{1}{2\epsilon^2}\log(\frac{2}{\delta})\epsilon^2\},\\
&=&2e^{-\log\big(\frac{2}{\delta}\big)},\\
&=&\delta.
\end{eqnarray*}
Hence the result.
\end{proof}
Lemma \ref{lemma:sample_size} provides a method to determine the sample size needed to keep the difference between the misclassification risk and the empirical misclassification risk to a small value, $\epsilon$,  with high probability $(1-\delta)$.

\subsection{Bayes Error Rate}
A review of the results of applying Algorithm \ref{algo:dt_tree_seg_reg} to various datasets used in the study reveals that the divide and conquer approach has some very useful implications in analyzing large datasets. It is evident that most problems are characterized by many regions where we achieve good success in predicting the class label and a few regions where predicting the class label is challenging. This characteristic is very useful because it points out the difficult regions of the dataset in terms of the classification task. Some questions that are of interest in the problematic regions of the dataset are the following: What is the best possible accuracy in these problematic regions? Are the features useful for the classification task in the problematic regions? The census income dataset is an example of where such questions are very relevant. The Bayes Error is a very useful theoretical idea to answer these questions. (see \cite{devroye96}[Chapter 2]). The optimal classifier $f^*(x)$ associated with a classification task is the Bayes classifier. For a binary classification problem, as is the case with the census income dataset, the Bayes classifier assigns class labels using the following rule:
$$ f^*(x)  =  \begin{cases*} 1\quad if\ \mathbb{P}\left[y = 1| x\right] > \frac{1}{2}\\ 0,\ otherwise. \end{cases*} $$

If we can estimate $\mathbb{P}\left[y = 1| x\right]$, we can estimate the performance of the Bayes classifier (see \cite{tumer2003bayes}). Density estimation is a computationally expensive task (see \cite{friedman2001elements}[Chapter 6, section 6.9]). Estimating the density for the entire dataset is computationally intractable. However, we are interested in evaluating this only for segments where we have poor classification accuracy. Since segments sizes are small and we want to perform this for a few segments only, this is computationally tractable. If it turns out that the Bayes classifier performs also poorly at these segments, then we know that the features are not useful for that segment and we need better features to improve classification accuracy. Applying these ideas to evaluate poorly performing segments is an area of future work. The intent here is to point out the localizing the problematic areas enables us to apply theoretical tools like the Bayes error rate to a small subset of our data. This makes such analysis more tractable than applying this to the entire dataset.

\section{Conclusion}\label{sec:conclusion}
\noindent We presented an an algorithm to perform classification tasks on large datasets. The algorithm uses a divide and conquer strategy to scale classification tasks. A decision tree is used to segment the dataset. By construction, many of the decision tree leaves are relatively homogeneous in terms of class distribution. Suitable classifiers can be used on the non-homogeneous leaves to determine class labels for the leaf instances. We demonstrated the effectiveness of this algorithm on large datasets. The algorithm achieves one of the following outcomes:
\begin{enumerate}
\item We achieve good classification accuracy. The levels of accuracy obtained is higher than what is achievable with a simple decision tree and can match the accuracy obtained using ensemble techniques like random forests or gradient boosted trees. Further the model produced by the proposed algorithm is easy to interpret and can yield insights related to the learning task. In contrast, though ensemble methods can be accurate, the models they produce are not very interpretable.
\item We are able to identify problematic or noisy regions of the dataset. This situation is characterized by decision tree nodes where all classifiers perform poorly. Typically this is a small portion of the dataset. This algorithm can be used to identify such regions of the dataset. These segments can then be isolated for further analysis. After removing these problematic segments, we are able to achieve high classification accuracy.
\end{enumerate} 
In summary, the proposed algorithm can produce models that are accurate and interpretable. This is highly desirable.

\bibliographystyle{ACM-Reference-Format}
\bibliography{ldcct} 

\end{document}